\documentclass[nohyperref]{article}

\usepackage{microtype}
\usepackage{graphicx}
\usepackage{subfigure}
\usepackage{booktabs} 
\usepackage{natbib}
\usepackage{hyperref}

\usepackage{stmaryrd}  

\usepackage[accepted]{icml2022}

\usepackage{amsmath}
\usepackage{amssymb}
\usepackage{mathtools}
\usepackage{amsthm}

\usepackage[capitalize,noabbrev]{cleveref}

\theoremstyle{plain}
\newtheorem{theorem}{Theorem}[section]
\newtheorem{proposition}[theorem]{Proposition}
\newtheorem{lemma}[theorem]{Lemma}

\theoremstyle{definition}

\theoremstyle{remark}
\newtheorem{remark}[theorem]{Remark}

\usepackage[textsize=tiny]{todonotes}

\icmltitlerunning{Neural Laplace for learning Stochastic Differential Equations}

\begin{document}

\twocolumn[
\icmltitle{Neural Laplace for learning Stochastic Differential Equations}



\icmlsetsymbol{equal}{*}

\begin{icmlauthorlist}
\icmlauthor{Adrien Carrel}{ic,cs}
\end{icmlauthorlist}

\icmlaffiliation{ic}{Department of Computing, Imperial College London, London, United Kingdom}
\icmlaffiliation{cs}{CentraleSupélec, Paris-Saclay University, Gif-sur-Yvette, France}

\icmlcorrespondingauthor{Adrien Carrel}{adrien.carrel22@imperial.ac.uk}

\icmlkeywords{Machine Learning, ICML}

\vskip 0.3in
]



\printAffiliationsAndNotice{\icmlEqualContribution} 

\begin{abstract}
Neural Laplace is a unified framework for learning diverse classes of differential equations (DE). For different classes of DE, this framework outperforms other approaches relying on neural networks that aim to learn classes of ordinary differential equations (ODE). However, many systems can't be modelled using ODEs. Stochastic differential equations (SDE) are the mathematical tool of choice when modelling spatiotemporal DE-dynamics under the influence of randomness. In this work, we review the potential applications of Neural Laplace to learn diverse classes of SDE, both from a theoretical and a practical point of view.

\end{abstract}

\section{Introduction}
\label{introduction}

Differential equations (DE) are important in many domains as they can model a wide variety of real-world situations. However, given a set of points sampled from a DE, it is difficult to predict the future trajectory without knowing the DE and its parameters. Many methods relying on neural networks have been developed to learn different classes of DE. Most of them have only been designed to predict Ordinary differential equations (ODE).

More recently, \cite{NeuralLaplace} developed a powerful framework based on the Laplace Transform for learning many classes of DE. However, this work has only been applied to certain types of DE. In this paper, we expand the scope of this method to another major class of DE, stochastic differential equations (SDE), to provide a more complete picture of this field. SDEs are used mostly in finance but may be applied in the foreseeable future in many other domains where randomness is part of the system, such as healthcare \cite{Breda2023}. In this paper, we investigate, in particular, the Geometric Brownian Motion.

Some applications of neural networks to learn SDE have been developed, like Neural SDE \cite{NSDE1,NSDE2}. However, they have been created to only solve this particular class of DE and may not be able to learn diverse classes of DE like Neural Laplace. Thus, in this paper, we focus on the potential applications of Neural Laplace to SDE to further confirm its \textit{unified framework} status.

\section{Problem formalism}
\label{problem}

\subsection{Goal}

After introducing some definitions from the Neural Laplace paper, we extend the Laplace Transform to stochastic processes. Then, we take a look at the well-known Geometric Brownian Motion and we show that the expectancy of its Laplace Transform can be calculated and that the variance of the norm of the Laplace Transform can be bounded. Therefore, under certain conditions, the norm of the Laplace Transform can be approximated with low uncertainty. As the imaginary part is driven by the imaginary part of the input (which is known), we can consider that the Laplace Transform is almost known and, therefore, that Neural Laplace can provide correct estimations after learning this class of DE.

\subsection{General definitions}

\textbf{Notation.} For a system with $D\in\mathbb{N}^{*}$ dimensions, we note for all $d\in\llbracket 1,D\rrbracket$ and for all $t\in\mathbb{R}$, the state of dimension $d$ at time $t$: $x_{d}(t)$. We define the function of time $x_{d}: \mathbb{R}\longrightarrow\mathbb{R}$, which is trajectory along the dimension $d$. $x_{d}(t)$ is the point on this trajectory at time $t$. Let $\textbf{x}(t)=\left [ x_{1}(t),\ldots,x_{D}(t)\right ]^{T}\in\mathbb{R}^{D}$ be the state vector at time $t$ and $\textbf{x}=\left [ x_{1},\ldots,x_{D}\right ]$ be the vector-valued trajectory. When conducting experiments later, we will consider a discrete set of observations $\mathcal{T}=\{ t_{1},\ldots,T\}$ to evaluate the state observations where $T=t_{N}$ and $N\in\mathbb{N}^{*}$ is the number of observations.

\textbf{Laplace Transform.} The Laplace transform of trajectory $\textbf{x}$ is defined as:

\begin{equation}
\textbf{F}(\textbf{s})=\mathcal{L}\{\textbf{x}\}(s)=\int_{0}^{\infty}e^{-\textbf{s}t}\textbf{x}(t)dt
\end{equation}

where $\textbf{s},\textbf{F}(\textbf{s})\in\mathbb{C}^{D}$. $\textbf{F}(\textbf{s})$, also called the Laplace representation, may have singularities ($\|\textbf{F}(\textbf{s})\| \rightarrow +\infty$ for some $\textbf{s}\in\mathbb{C}^{D}$).

\textbf{Stochastic Differential Equations.} A SDE is a differential equation in which one or more of the terms is a stochastic process. The solution to this type of equation is also a stochastic process.

\textbf{Stochastic process.} Let $(\Omega,\mathcal{F},\mathbb{P})$ be a probability space and $(E,\mathcal{E})$ be a measurable space. A stochastic process is a collection of random variables $(X_{t})_{t\in \mathcal{T}}$ indexed by a set $\mathcal{T}$ where for all $t\in \mathcal{T}$, $X_{t}: (\Omega,\mathcal{F},\mathbb{P})\longrightarrow (E,\mathcal{E})$.

\textbf{Brownian motion.} Let $(\Omega,\mathcal{F},\mathbb{P},(\mathcal{F}_{t})_{t\in\mathbb{R}_{+}})$ be a probability space with a filtration. A standard Brownian motion (or standard Wiener process) is a continuous centered gaussian process $\{B_{t},t\in\mathbb{R}_{+}\}$ such that $\forall s,t\in\mathbb{R}_{+}$, $Cov(B_{t},B_{s})=min(t,s)$, and for all $t\geq 0$, $B_{t}$ is $\mathcal{F}_{t}$-measurable.

In particular, a Brownian motion verifies the following properties:

\begin{itemize}
\item $B_{0}=0$ (standard)
\item $t\mapsto B_{t}(w)$ is continuous for almost all $w\in\Omega$
\item For all $s,t\in\mathbb{R}_{+}$ such that $s<t$, $B_{t}-B_{s}\sim \mathcal{N}(0,t-s)$ (gaussian distribution with mean 0 and variance $t-s$)
\item Independent increments: for all $n\in\mathbb{N}^{*}$, for all $t_{1},\ldots,t_{n}\in\mathbb{R}_{+}$ such that $0=t_{0}<t_{1}<\ldots<t_{n}$, the random variables $(B_{t_{k}}-B_{t_{k-1}})_{k\in\llbracket 1, n\rrbracket}$ are independent.
\end{itemize}

This stochastic process is often used to model randomness in SDEs, including the Geometric Brownian Motion.

\textbf{Geometric Brownian Motion.} Let $(\Omega,\mathcal{F},\mathbb{P},(\mathcal{F}_{t})_{t\in\mathbb{R}_{+}})$ be a probability space with a filtration and $X=(X_{t})_{t\in\mathbb{R}_{+}}$ be a stochastic process that takes value in $(\mathbb{R},\mathcal{B}(\mathbb{R}))$.

$X$ is a Geometric Brownian Motion (GBM) if it is a solution of the SDE:

\begin{equation}
dX_{t}=\mu X_{t} dt + \sigma X_{t} B_{t}
\end{equation}

where $\mu\in\mathbb{R}$ and $\sigma \geq 0$ are two constants and $B=(B_{t})_{t\in\mathbb{R}_{+}}$ is a Brownian motion.

Using Itô's calculus, one can calculate an explicit solution for GBMs:

\begin{equation}
X_{t} = X_{0} e^{\left (\mu - \frac{\sigma^{2}}{2}\right )t + \sigma B_{t}}
\end{equation}

\subsection{Properties of Geometric Brownian Motions}

\begin{lemma}
\label{lemma:gmb}
If $Z$ is a standard normal distribution ($Z\sim\mathcal{N}(0, 1)$), then $\forall \lambda\in\mathbb{R}, \mathbb{E}\left [ e^{\lambda Z}\right ]=e^{\frac{\lambda^{2}}{2}}$
\end{lemma}
\begin{proof}
Let $Z$ be a random variable such that $Z\sim\mathcal{N}(0, 1)$. For all $\lambda\in\mathbb{R}$,

\begin{align*}
\mathbb{E}\left [ e^{\lambda Z}\right ] &= \int_{-\infty}^{+\infty}e^{\lambda x}\frac{1}{\sqrt{2\pi}}e^{-\frac{x^{2}}{2}}dx\\ 
 &= \int_{-\infty}^{+\infty}\frac{1}{\sqrt{2\pi}}e^{-\frac{(x-\lambda)^{2}}{2}}dx\\ 
 &= e^{\frac{\lambda^{2}}{2}}\int_{-\infty}^{+\infty}\frac{1}{\sqrt{2\pi}}e^{-\frac{x^{2}}{2}}dx\\
 &= e^{\frac{\lambda^{2}}{2}}
\end{align*}
\end{proof}

\begin{proposition}
\label{prop:gbm}
If $X=(X_{t})_{t\in\mathbb{R}_{+}}$ is a GBM, then for all $t\geq 0$:
\begin{itemize}
\item $\mathbb{E}[X_{t}]=X_{0}e^{\mu t}$
\item $\mathbb{V}[X_{t}]=X^{2}_{0}e^{2\mu t}\left ( e^{\sigma^2 t} - 1\right )$
\end{itemize}
\end{proposition}
\begin{proof} 
Let $X$ be a GBM and $t\in\mathbb{R}_{+}$.

1) \begin{align*}
\mathbb{E}[X_{t}] &= \mathbb{E}\left [ X_{0} e^{\left (\mu - \frac{\sigma^{2}}{2}\right )t + \sigma B_{t}} \right ]\\ 
 &= X_{0} e^{\left (\mu - \frac{\sigma^{2}}{2}\right )t}\mathbb{E}\left [ e^{\sigma B_{t}} \right ]\\
 &= X_{0} e^{\left (\mu - \frac{\sigma^{2}}{2}\right )t}\mathbb{E}\left [ e^{\sigma \sqrt{t}B_{1}} \right ]
\end{align*}

$B_{1}\sim\mathcal{N}(0, 1)$ so using \cref{lemma:gmb}, we obtain that $\mathbb{E}\left [ e^{\sigma \sqrt{t}B_{1}} \right ]=e^{\frac{\sigma^{2}t}{2}}$, hence the result.

2) $\mathbb{V}[X_{t}] = \mathbb{E}\left [ X^{2}_{t}\right ] - \mathbb{E}\left [ X_{t}\right ]^{2}$, and

\begin{align*}
\mathbb{E}[X^{2}_{t}] &= \mathbb{E}\left [ X^{2}_{0} e^{\left (2\mu - \sigma^{2}\right )t + 2\sigma B_{t}} \right ]\\ 
 &= X^{2}_{0} e^{\left (2\mu - \sigma^{2}\right )t}\mathbb{E}\left [ e^{2\sigma B_{t}} \right ]
\end{align*}

Applying \cref{lemma:gmb} again, we obtain $\mathbb{E}\left [ e^{2\sigma B_{t}} \right ]=e^{2\sigma^{2} t}$. Thus,

\begin{align*}
\mathbb{V}[X_{t}] &= X^{2}_{0}e^{\left (2\mu + \sigma^{2}\right )t}-\left ( X^{2}_{0}e^{\mu t}\right )^{2}\\
 &= X^{2}_{0}e^{2\mu t}\left ( e^{\sigma^2 t} - 1\right )
\end{align*}
\end{proof}

\textbf{Laplace Transform of a real continuous stochastic process.} To apply Neural Laplace to a SDE, we first need to define the Laplace Transform of a stochastic process. Let $X=(X_{t})_{t\in\mathbb{R}_{+}}$ be a stochastic process defined on a probability space $(\Omega,\mathcal{F},\mathbb{P},(\mathcal{F}_{t})_{t\in\mathbb{R}_{+}})$, that takes values in $(\mathbb{R}^{D},\mathcal{B}(\mathbb{R}^{D}))$ where $D>0$ is the dimension, and such that $t\mapsto X_{t}$ is almost surely continuous.

We can extend the Laplace Transform to a stochastic process by defining a function $\textbf{F}$ such that, for all $\textbf{s}\in\mathbb{C}^{D}$, $\textbf{F}(\textbf{s})$ is a stochastic process defined on the same probability space, that takes values in $(\mathbb{C}^{D},\mathcal{B}(\mathbb{C}^{D}))$, and that verifies:

$$\forall\omega\in\Omega, \textbf{F}(\textbf{s})(\omega)=\mathcal{L}\{X(\omega)\}(s)=\int_{0}^{\infty}e^{-\textbf{s}t}X_{t}(\omega)dt$$

We will note $\textbf{F}(\textbf{s})=\int_{0}^{\infty}e^{-\textbf{s}t}X_{t}dt$.

\subsection{Laplace Transform of a Geometric Brownian Motion}

\begin{theorem}
\label{theorem:gbm}
Let $X$ be a GBM with parameters $\mu\in\mathbb{R},\sigma\geq 0$. For all $s\in\mathbb{C}$ verifying $\Re (s)>\mu$, we have

$$\mathbb{E}\left [ F(s)\right ]=\frac{X_{0}}{s-\mu}$$

Moreover, if $s$ also verifies $\Re (s)>0$ and if $\mu -\frac{\sigma^{2}}{2}>0$, then

$$\mathbb{V}\left [ \|F(s)\|\right ]\leq\frac{\|X_{0}\|^{2}}{4\Re(s)(\mu-\frac{\sigma^{2}}{2})}$$

\end{theorem}
\begin{proof}
Let $X=(X_{t})_{t\in\mathbb{R}_{+}}$ be a Geometric Brownian Motion with parameters $\mu\in\mathbb{R},\sigma\geq 0$, and $s\in\mathbb{C}$.

1) \begin{align*}
\mathbb{E}[F(s)] &= \mathbb{E}\left [ \int_{0}^{+\infty}X_{t}e^{-st}dt \right ]\\
 &= \int_{0}^{+\infty}\mathbb{E}\left [ X_{t}e^{-st} \right ]dt\text{ (*)}\textsuperscript{1}\\
 &= \int_{0}^{+\infty} X_{0}e^{(\mu-s)t}dt\\
 &= \frac{X_{0}}{s-\mu}
\end{align*}

*: We can apply Fubini's theorem because $\int_{X\times T}|f(x,y)|d(x,y)<+\infty$ and $f$ is measurable, where $X=[-\infty,+\infty [$, $T=[0,+\infty [$, and $f:(x,t)\mapsto X_{0}e^{\left (\mu - s - \frac{\sigma^{2}}{2} \right )t + \sigma \sqrt{t} x}\frac{1}{\sqrt{2\pi}}e^{-\frac{x^{2}}{2}}$

2) We now that $\mathbb{V}\left [ \|F(s)\| \right ]\leq \mathbb{E}\left [ \|F(s)\|^{2} \right ]$, Now let's find an upper bound of $\mathbb{E}\left [ \|F(s)\|^{2} \right ]$.

\begin{align*}
\mathbb{E}\left [ \|F(s)\|^{2} \right ] &= \mathbb{E}\left [ \left \| \int_{0}^{+\infty}X_{t}e^{-st}dt \right \|^{2} \right ]\\
 &\leq \mathbb{E}\left [ \left (\int_{0}^{+\infty}\|X_{t}e^{-st}\|dt \right )^{2} \right ]\\
 &\leq \mathbb{E}\bigg[ \left (\int_{0}^{+\infty}\|X_{0}\|^{2}e^{-2\Re(s)t}dt \right )\times\\
 & \left (\int_{0}^{+\infty}e^{-\sigma^{2}t-2\mu t+2\sigma\sqrt{t}B_{1}}dt \right ) \bigg]\text{ (**)}\textsuperscript{2}\\
 &\leq \frac{\|X_{0}\|^{2}}{2\Re(s)}\int_{0}^{+\infty}e^{-\sigma^{2}t-2\mu t}\mathbb{E}\left [ e^{2\sigma\sqrt{t}B_{1}} \right ]dt\text{ (*)}\textsuperscript{1}\\
 &\leq \frac{\|X_{0}\|^{2}}{2\Re(s)}\int_{0}^{+\infty} e^{(\sigma^{2}-2\mu)t}dt\\
 &\leq \frac{\|X_{0}\|^{2}}{4\Re(s)(\mu-\frac{\sigma^{2}}{2})}
\end{align*}
\end{proof}
\footnotetext{\textsuperscript{1}*: Fubini's Theorem}
\footnotetext{\textsuperscript{2}**: Cauchy-Schwarz inequality}

\section{Proposed method}
\label{method}
According to \cref{theorem:gbm}, for a stochastic process $X$ solution of the GBM SDE with a low initial value $|X_{0}|$, a low volatility $\sigma\geq 0$, and a high drift $\mu$, we can expect that for most $s$ in the complex domain $\mathbb{C}$, $F(s)$ is well-defined and $F(s)\approx\mathbb{E}\left [ F(s)\right ]$. This is coherent with the fact that $X$ is essentially driven by the drift and not really affected by the Brownian motion, thus limiting the randomness.

\textbf{Experiment.} We generate $N_{trajectories}=200$ trajectories with $N_{samples}=200$ points each sampled uniformly in $[\frac{T}{N_{samples}}, T]$ where $T=1$. For all $i\in\llbracket 1, N_{trajectories} \rrbracket$, a trajectory $\tau_{i}$ is characterized by its three parameters $X_{0,i},\mu_{i},\sigma_{i}$ where $X_{0,i}\sim\mathcal{U}(0.1, 1)$, $\mu_{i}\sim\mathcal{U}(4, 8)$ and $\sigma_{i}\sim\mathcal{U}(0.1, 1)$.

We run 100 epochs to train the Neural Laplace and all the other baseline models. The Root Mean Squared Error (RMSE) will be evaluated on a test set for 3 different seeds and the mean and the standard deviation (std) across the different training are reported in the table \cref{table:results} in the \textit{Results} section below.

\begin{remark}
\label{remark:mu}
Due to the choice of parameters to generate the trajectories, especially the drift $\mu$, $X_{t}\approx e^{\mu t}$, so we can expect that the other baseline models will also perform well as $X$ is approximately the solution of the ODE: $\frac{dX}{dt}(t)-\mu X = 0$.
\end{remark}

\section{Results}
\label{results}

\begin{table}[t]
\caption{Mean and Std of the test RMSE across 3 different seeds for the different methods.}
\label{table:results}
\vskip 0.15in
\begin{center}
\begin{small}
\begin{sc}
\begin{tabular}{lcccr}
\toprule
Method & Mean & Std \\
\midrule
ANODE (euler) & 0.387012 & 0.025194 \\
Latent ODE (ODE enc.) & 0.529489 & 0.074544 \\
NODE (euler) & 0.383214 & 0.033135 \\
Neural Flow Coupling & 0.577819 & 0.075091 \\
Neural Flow ResNet & 0.605377 & 0.037672 \\
Neural Laplace & 0.597754 & 0.095084 \\
\bottomrule
\end{tabular}
\end{sc}
\end{small}
\end{center}
\vskip -0.1in
\end{table}

Some of the baseline methods seem to perform better than Neural Laplace to predict future values of a Geometric Brownian Motion (see \cref{table:results}). However, this can be explained with the remark \cref{remark:mu}.

Nonetheless, the mean of the test RMSE for the Neural Laplace method is still low, which highlights the fact the method works.

\section{Conclusion}
\label{conclusion}
This short piece of work extended the definition of the Laplace Transform to stochastic processes in order to apply the Neural Laplace method to solve stochastic differential equations. From a theoretical point of view, we established conditions that a Geometric Brownian Motion must satisfy to potentially lead to better predictions when training a Neural Laplace method. Then, we conducted an experiment to show that Neural Laplace can learn from this class of differential equation in this particular case, and we compared the results with baseline models that work for ODE.

\section{Future work}
\label{future}
Among all the stochastic differential equations, we only investigated the Geometric Brownian Motion. Looking at other SDEs to assess the feasibility of using Neural Laplace in such settings both theoretically and empirically would be the first things to add to this work. Let alone the fact that GBM is not the best example, as highlighted by \cref{remark:mu}. Moreover, for GBM, testing other sets of values for $\mu$ and $\sigma$ may lead to some other results that could be investigated.

When comparing Neural Laplace with baseline methods, we didn't add Neural Stochastic Differential Equations methods.

Finally, as the Laplacian Transform of a stochastic process is a random variable, we could research how to provide confidence intervals to the predictions made by the Neural Laplace methods based on the uncertainty that we have on $F(s)$.

\bibliography{report}

\begin{thebibliography}{4}
\providecommand{\natexlab}[1]{#1}
\providecommand{\url}[1]{\texttt{#1}}
\expandafter\ifx\csname urlstyle\endcsname\relax
  \providecommand{\doi}[1]{doi: #1}\else
  \providecommand{\doi}{doi: \begingroup \urlstyle{rm}\Url}\fi

\bibitem[Breda \& Canci(2023)Breda and Canci]{Breda2023}
Breda, D. and Canci, Jung Kyuand~D'Ambrosio, R.
\newblock \emph{An Invitation to Stochastic Differential Equations in
  Healthcare}, pp.\  97--110.
\newblock Springer International Publishing, Cham, 2023.
\newblock ISBN 978-3-031-11814-2.
\newblock \doi{10.1007/978-3-031-11814-2_6}.
\newblock URL \url{https://doi.org/10.1007/978-3-031-11814-2_6}.

\bibitem[Holt et~al.(2022)Holt, Qian, and van~der Schaar]{NeuralLaplace}
Holt, S., Qian, Z., and van~der Schaar, M.
\newblock Neural laplace: Learning diverse classes of differential equations in
  the laplace domain, 2022.

\bibitem[Salvi et~al.(2022)Salvi, Lemercier, and Gerasimovics]{NSDE2}
Salvi, C., Lemercier, M., and Gerasimovics, A.
\newblock Neural stochastic pdes: Resolution-invariant learning of continuous
  spatiotemporal dynamics, 2022.

\bibitem[Yang et~al.(2022)Yang, Gao, Lu, Duan, and Liu]{NSDE1}
Yang, L., Gao, T., Lu, Y., Duan, J., and Liu, T.
\newblock Neural network stochastic differential equation models with
  applications to financial data forecasting, 2022.

\end{thebibliography}

\end{document}